\documentclass{article}

\usepackage{arxiv}

\usepackage[utf8]{inputenc} 
\usepackage[T1]{fontenc}    
\usepackage{hyperref}       
\usepackage{url}            
\usepackage{booktabs}       
\usepackage{amsfonts}       
\usepackage{nicefrac}       
\usepackage{microtype}      
\usepackage{lipsum}
\usepackage{amsthm}
\usepackage{graphicx}
\usepackage{acronym}
\usepackage{wrapfig}
\usepackage{psfrag}
\usepackage{cite}
\usepackage{multirow}
\usepackage{subfigure}
\usepackage{mathtools}
\usepackage[utf8]{inputenc} 
\usepackage[T1]{fontenc}    
\usepackage{hyperref}       
\usepackage{url}            
\usepackage{booktabs}       
\usepackage{amsfonts}       
\usepackage{nicefrac}       
\usepackage{microtype}      
\usepackage{mathtools}
\usepackage{txfonts}
\usepackage{relsize}%
\usepackage{acronym}
\usepackage{amsmath}
\usepackage{amssymb}
\usepackage{amsthm}
\usepackage{mathtools}
\usepackage{amsfonts}
\usepackage{calc}
\usepackage{graphicx}
\usepackage{color}
\usepackage{psfrag}
\usepackage{pgfplots}
\usepackage{caption}
\usepackage{xpatch}
\usepackage{enumerate}
\usepackage{tikz}
\usepackage{fancyvrb}
\usepackage{algorithm,algpseudocode}
\usepackage{acronym}

\theoremstyle{definition}

\newtheorem{proposition}{Proposition}
\newtheorem{theorem}{Theorem}

\newtheorem{lemma}{Lemma}

\newcommand\V[1]  { \mathbf{#1} }
\newcommand\B[1]  { \boldsymbol{#1} }

\newcommand\set[1] {\mathcal{#1}}
\acrodef{LPC}{linear probabilistic classifier}
\acrodef{RV}{random variable}
\acrodef{ERM}{empirical risk minimization}
\acrodef{RRM}{robust risk minimization}
\acrodef{SVM}{support vector machine}
\acrodef{ANN}{artificial neural network}
\acrodef{RKHS}{reproducing kernel Hilbert space}
\acrodef{MEM}{maximum entropy machine}
\acrodef{DT}{decision tree}
\acrodef{QDA}{quadratic discriminant analysis}
\acrodef{NN}{nearest neighbor} 
\acrodef{RF}{random forest}
\acrodef{LR}{logistic regression}
\acrodef{LUSI}{learning using statistical invariants}
\acrodef{ACSC}{adversarial cost-sensitive classifier}

\title{Supervised classification via minimax\\ probabilistic transformations}
\author{
  Santiago Mazuelas \\
  BCAM-Basque Center of Applied Mathematics\\
  Bilbao, Spain\\
  \texttt{smazuelas@bcamath.org} \\
  \And
  Andrea Zanoni\\
  \'{E}cole Polytechnique F\'{e}d\'{e}rale de Lausanne\\  
  Lausanne, Switzerland\\
  \texttt{andrea.zanoni@epfl.ch }
   \And
 Aritz P\'{e}rez\\
  BCAM-Basque Center of Applied Mathematics\\
  Bilbao, Spain\\
  \texttt{aperez@bcamath.org} \\
}
\date{}

\begin{document}

\maketitle
\begin{abstract}
Conventional techniques for supervised classification constrain the classification rules considered and use surrogate losses for classification 0-1 loss. Favored families of classification rules are those that enjoy parametric representations suitable for surrogate loss minimization, and low complexity properties suitable for overfitting control. This paper presents classification techniques based on robust risk minimization (RRM) that we call linear probabilistic classifiers (LPCs). The proposed techniques consider unconstrained classification rules, optimize the classification 0-1 loss, and provide performance bounds during learning. LPCs enable efficient learning by using linear optimization, and avoid overffiting by using RRM over polyhedral uncertainty sets of distributions. We also provide finite-sample generalization bounds for LPCs and show their competitive performance with state-of-the-art techniques using benchmark datasets.
\end{abstract}


\section{Introduction}

Supervised classification uses training data to find a classification rule with small risk (out-of-sample error). Risk minimization cannot be addressed in practice since the probability distribution of pairs features-label is unknown. Therefore, learning techniques for supervised classification obtain classification rules by addressing a surrogate for risk minimization. The most common surrogate is \ac{ERM} that is based on 
minimizing the loss achieved with training examples. Such approach may suffer from overfitting, usually addressed by constraining the classification rules considered to have reduced complexity \cite{Vap:98,EvgPonPog:00}. Other surrogate for 
risk minimization is \ac{RRM} that is based on minimizing the worst-case risk against a set of probability distributions consistent with 
training data. Such approach avoids overfitting as long as the probability distribution of features-label pairs 
belongs to the uncertainty set considered, but it requires to solve a minimax optimization problem \cite{LanGhaBhaJor:02,DelYe:10,FarTse:16,AsiXinBeh:15,AbaMohKuh:15,ShaKuhMoh:17,DucGlyNam:16,NamDuc:17,LeeRag:18}.

Conventional learning techniques for supervised classification constrain the classification rules considered and use surrogate losses. Favored families of classification rules are those that enjoy parametric representations suitable for surrogate loss minimization, and low complexity properties suitable for overfitting control. Techniques based on regularization in \acp{RKHS} \cite{Vap:98}, such as \acp{SVM} and kernel logistic regression, consider classification rules obtained from functions with reduced norm in an \ac{RKHS}, with different design choices such as kernel and regularization parameters. Techniques based on \acp{ANN} \cite{GooBenCou:16} consider classification rules with a hierarchical structure, with different design choices such as network architecture and activation functions. Techniques based on ensemble learning \cite{SchFre:12}, such as Adaboost and \acp{RF}, consider classification rules obtained by combinations of weak rules, with different design choices such as type of weak rules and aggregation method. In addition, conventional techniques enable tractable optimization of the classification rule's parameters by using a surrogate loss (e.g., hinge, logistic, cross-entropy, and exponential) instead of the original target given by classification 0-1 loss.


\textbf{Main contributions}

This paper presents techniques for supervised classification based on \ac{RRM} that we call \acp{LPC}. The proposed techniques consider unconstrained classification rules, optimize the classification 0-1 loss, and provide performance bounds during learning. 
Current techniques based on \ac{RRM} utilize uncertainty sets of distributions similar to the empirical distribution in terms of several metrics such as moments and marginals fits \cite{LanGhaBhaJor:02,DelYe:10,FarTse:16,AsiXinBeh:15}, Wasserstein distances \cite{AbaMohKuh:15,ShaKuhMoh:17}, and f-divergences \cite{DucGlyNam:16,NamDuc:17}. The proposed \acp{LPC} utilize uncertainty sets of distributions given by constraining the expectations of a chosen function that we call generating function. Such distributions are similar in terms of the probability metric given by the generating function \cite{Mul:97}. Most \ac{RRM} methods enable efficient minimax optimization by using parametric families of classification rules and surrogate losses. Techniques based on Wasserstein distances use linear functions or \acp{RKHS} and surrogate log lossess \cite{AbaMohKuh:15,ShaKuhMoh:17}, while techniques based on f-divergences can use more general parametric families of classification rules and surrogate losses as long as they result in convex losses \cite{DucGlyNam:16,NamDuc:17}. As the proposed \acp{LPC}, techniques in \cite{AsiXinBeh:15} consider unconstrained classification rules exploiting Lagrange duality. Such work uses uncertainty sets defined by equality constraints, and its learning stage is enabled by approximate optimization with a stochastic gradient descent algorithm. On the other hand, \acp{LPC} consider uncertainty sets that contain the actual distribution with a tunable confidence, and \acp{LPC} learning is enabled by the reformulation of minimax problem as a linear program.


More detailed comparisons with related techniques are provided in the remarks to the paper's main new results, organized as follows:
\begin{itemize}
\item Learning techniques that determine \acp{LPC} as the solution of a linear optimization problem (Theorem~\ref{th1} in Section~\ref{sec-2}). 
\item Techniques that obtain upper and lower bounds for the expected loss of general classification rules (Theorem~\ref{th1} and Proposition~\ref{prop} in Section~\ref{sec-2}). 
\item Finite-sample generalization bounds for the risk of \acp{LPC} in terms of training size and parameters describing the complexity of the generating function (Theorem~\ref{th-bounds} in Section~\ref{sec-3}).
\end{itemize}
In addition, Section~\ref{sec-4} describes efficient implementations for \acp{LPC} and proposes a simple generating function, and Section~\ref{sec-5} shows the suitability of the presented performance' bounds and compares the classification error of \acp{LPC} with respect to state-of-the-art techniques.

\emph{Notation:} calligraphic upper case letters denote sets; real-valued functions and vector-valued functions are denoted by lower and upper case letters, respectively; vectors and matrices are denoted by bold lower and upper case letters, respectively; $\V{v}^{\text{T}}$, $\V{v}^+$, and $\|\V{v}\|_{q,r}$ denote the transpose, positive part, and $(q,r)$-mixed norm of vector $\V{v}$,\footnote{The $(q,r)$-mixed norm of a vector $\V{v}\in\mathbb{R}^{I\cdot J}$ indexed by $\{1,2,\ldots,I\}\times\{1,2,\ldots,J\}$ is $\|\V{v}\|_{q,r}=\|\left[\|\V{v}_1\|_q,\|\V{v}_2\|_q,\ldots,\|\V{v}_I\|_q\right]^{\text{T}}\|_r$ where $\V{v}_i=[v_{(i,1)},v_{(i,2)},\ldots,v_{(i,J)}]^{\text{T}}\in\mathbb{R}^J$ for $i=1,2,\ldots,I$. For instance, $\|\V{v}\|_{1,\infty}=\max_{i\in\set{I}}\sum_{j\in\set{J}}|\V{v}_{(i,j)}|$.} respectively; $\mathbb{E}_p\{\cdot\}$ denotes expectation with respect to probability distribution $p$; $\preceq$ and $\succeq$ denote vector (component-wise) inequalities; $\V{1}$ denotes a vector with all components equal to $1$; and $|\set{Z}|$ denotes de cardinality of set $\set{Z}$. We represent real-valued and vector-valued functions with finite domains by vectors and matrices, respectively; specifically, we represent a function $f:\set{Z}\to\mathbb{R}$ for finite set $\set{Z}=\{z_1,z_2,\ldots,z_k\}$ with vector $\V{f}=[f(z_1),f(z_2),\ldots,f(z_k)]^\text{T}\in\mathbb{R}^k$, and a vector function $F:\set{Z}\to\mathbb{R}^m$ by matrix $\V{F}\in\mathbb{R}^{m\times k}$ with column $i$ given by $F(z_i)$ for $i=1,2,\ldots,k$. In addition, if $F$ is a function with domain $\set{X}\times\set{Y}$, for each $x\in\set{X}$ we represent by $F_x$ the function $F_x(y)=F(x,y)$ with domain $\set{Y}$. Finally, we denote by $\Delta(\set{Z})$ the set of probability distributions with support $\set{Z}$ and represent each $p\in\Delta(\set{Z})$ for finite set $\set{Z}$ by its probability mass function $p:\set{Z}\to\mathbb{R}$ with $\V{p}\succeq \V{0}$ and $\V{p}^{\text{T}}\V{1}=1$.
\section{Minimax classification over polyhedral uncertainty sets}\label{sec-2}
This section first briefly describes the problem statement for supervised classification, and then presents techniques to learn \acp{LPC} and to bound expected losses.
In what follows, features and labels are elements of sets $\set{X}$ and $\set{Y}$, respectively. We assume that both sets are finite; commonly the cardinality of $\set{X}$ is very large while that of $\set{Y}$ is very small. Such finiteness assumption does not lose any generality in practice, at least using digital computers.

A deterministic classification rule is a function from $\set{X}$ to $\set{Y}$. In this paper we consider also classification rules that are allowed to randomly classify each feature, so that a general classification rule is given by a probabilistic transformation also known as Markov transition or channel \cite{RooWill:18}. We denote by $\Delta(\set{X},\set{Y})$ the set of probabilistic transformations from $\set{X}$ to $\set{Y}$, that is, functions from $\set{X}$ to $\Delta(\set{Y})$. In what follows we represent each $h\in\Delta(\set{X},\set{Y})$ by a Markov kernel function $h: \set{X}\times\set{Y}\to\mathbb{R}$ with $h_x$ a probability mass function in $\set{Y}$ for any $x\in\set{X}$ (i.e., $\V{h}_x\succeq \V{0}$ and $\V{h}_x^{\text{T}}\V{1}=1$). A classification rule $h\in\Delta(\set{X},\set{Y})$ classifies each feature $x\in\set{X}$ as label $y\in\set{Y}$ with probability $h(x,y)$. In particular, deterministic classification rules correspond to $h\in\Delta(\set{X},\set{Y})$ that takes only values $0$ and $1$.



The classification $0$-$1$ loss (called just loss in the following) of a classification rule at $(x,y)\in\set{X}\times\set{Y}$ is $0$ if it classifies $x$ with $y$, and is $1$ otherwise. Hence, the expected loss of classification rule $h\in\Delta(\set{X},\set{Y})$ with respect to a probability distribution $p\in\Delta(\set{X}\times\set{Y})$ is $$\ell(h,p)=1-\V{p}^{\text{T}}\V{h}.$$
 The risk of a classification rule $h$ (denoted $R(h)$) is its expected loss with respect to the actual distribution of features-label pairs $p^*$, that is
 $$R(h)=\ell(h,p^*).$$ The minimum risk is known as Bayes risk and becomes $R_{\text{Bayes}}= 1-\|\V{p}^*\|_{\infty,1}$ since it is achieved by Bayes' rule $h_{\text{Bayes}}$ that classifies each $x\in\set{X}$ with a label attaining the maximum of $p_x^*$.


The goal of supervised classification is to determine a classification rule with reduced risk by using a set of training samples. \ac{ERM} approach is based on minimizing the empirical risk $\ell(h,p_n)$, where $p_n$ is the empirical distribution of training samples \cite{Vap:98,EvgPonPog:00}. \ac{RRM} approach is based on minimizing the maximum (worst-case) risk $\ell(h,p)$ for $p$ a probability distribution in an uncertainty set obtained from training samples \cite{LanGhaBhaJor:02,DelYe:10,FarTse:16,AsiXinBeh:15,AbaMohKuh:15,ShaKuhMoh:17,DucGlyNam:16,NamDuc:17,LeeRag:18}. The following shows how uncertainty sets defined by linear inequalities enable efficient \ac{RRM} without constraining the set of classification rules. 

Given vectors $\V{a},\V{b}\in\mathbb{R}^m$ with $\V{a}\preceq\V{b}$ and a vector function $\Phi:\set{X}\times\set{Y}\to\mathbb{R}^m$, we denote by $\set{U}_\Phi^{\V{a},\V{b}}$ the set
$$\set{U}_\Phi^{\V{a},\V{b}}=\{p\in\Delta(\set{X}\times\set{Y}):\  \V{a}\preceq\mathbb{E}_p\{\Phi(x,y)\}\preceq \V{b}\}.$$
In addition, we call function $\Phi$ the generating function, and vectors $\V{a}$ and $\V{b}$ the lower and upper endpoints of  expectation interval estimates. The minimax expected loss against uncertainty set $\set{U}_\Phi^{\V{a},\V{b}}$ is
\begin{align}\label{minimax-risk}R_\Phi^{\V{a},\V{b}}=\min_{h\in\Delta(\set{X},\set{Y})}\max_{p\in\set{U}_\Phi^{\V{a},\V{b}}}\ell(h,p)=1-\min_{p\in\set{U}_\Phi^{\V{a},\V{b}}}\|p\|_{\infty,1}\end{align}
where the second equality is obtained since the minimax coincides with the maximin because $\Delta(\set{X},\set{Y})$ and $\set{U}_\Phi^{\V{a},\V{b}}$ are closed convex sets of $\mathbb{R}^{|\set{X}||\set{Y}|}$ \cite{GruDaw:04}.
In the following, whenever we use an expectation point estimate, i.e., $\V{a}=\V{b}$, we drop $\V{b}$ from the superscripts, for instance we denote $\set{U}_\Phi^{\V{a},\V{b}}$ for $\V{a}=\V{b}$ as $\set{U}_\Phi^{\V{a}}$.

Uncertainty sets $\set{U}_\Phi^{\V{a},\V{b}}$ are polyhedra in $\Delta(\set{X}\times\set{Y})\subset\mathbb{R}^{|\set{X}||\set{Y}|}$ defined by affine inequality constraints since $\mathbb{E}_p\{\Phi(x,y)\}=\B{\Phi}\V{p}$. They contain probability distributions that are similar in terms of the generating function's expectations, for instance, two distributions are in the same uncertainty set $\set{U}_\Phi^{\V{a}}$ for some $\V{a}\in\mathbb{R}^m$ if their distance is zero for the semi-metric generated by $\Phi$ \cite{Mul:97}.

The following result determines minimax classification rules against the above uncertainty sets as well as the corresponding minimax expected loss.   
\begin{theorem}\label{th1}
Let $\Phi:\set{X}\times\set{Y}\to\mathbb{R}^m$ and $\V{a}, \V{b}\in\mathbb{R}^{m}$ with $\V{a}\preceq\V{b}$. If a classification rule
$h^*\in\Delta(X,Y)$ satisfies 
\begin{align}\label{robust-act}\V{h}^*\succeq \B{\Phi}^{\text{T}}(\B{\alpha}^*-\B{\beta}^*)+\V{1}\gamma^*\end{align}
for $\B{\alpha}^*,\B{\beta}^*$, $\gamma^*$ solution of optimization problem
\begin{align}\label{learning-ineq}\begin{array}{cc}\underset{\gamma\in\mathbb{R},\B{\alpha},\B{\beta}\in\mathbb{R}^{m}}{\max}&\V{a}^{\text{T}}\B{\alpha}-\V{b}^{\text{T}}\B{\beta}+\gamma\\
\mbox{s. t.}&\|(\B{\Phi}^{\text{T}}(\B{\alpha}-\B{\beta})+\V{1}\gamma)^+\|_{1,\infty}\leq 1\\
&\B{\alpha},\B{\beta}\succeq \V{0} \end{array}\end{align}
then, $h^*$ is a minimax classification rule againts uncertainty set $\set{U}_{\Phi}^{\V{a},\V{b}}$, that is,
$$h^*\in\arg\min_{h\in\Delta(X,Y)}\max_{p\in\set{U}_{\Phi}^{\V{a},\V{b}}}\ell(h,p).$$
In addition, the minimax expected loss against uncertainty set $\set{U}_\Phi^{\V{a},\V{b}}$ is given by \begin{align}\label{upper}R_\Phi^{\V{a},\V{b}}=1-\V{a}^{\text{T}}\B{\alpha}^*+\V{b}^{\text{T}}\B{\beta}^*-\gamma^*.\end{align}
\end{theorem}
\vspace{-0.3cm}
\begin{proof}
See Appendix~\ref{proof-th1}.
\end{proof}
\vspace{-0.2cm}
Classification rules satisfying \eqref{robust-act} always exist since for any $x\in\set{X}$, $\|\left(\B{\Phi}_x^{\text{T}}(\B{\alpha}^*-\B{\beta}^*)+\V{1}\gamma^*\right)^+\|_1\leq 1$ due to the constraints in \eqref{learning-ineq}. In addition, a classification rule satisfying \eqref{robust-act} can be directly obtained from a solution of \eqref{learning-ineq} $\B{\alpha}^*,\B{\beta}^*,\gamma^*$ as 
\begin{align}\label{t-a,b}
h^{\V{a},\V{b}}(x,y)=(\Phi(x,y)^{\text{T}}(\B{\alpha}^*-\B{\beta}^*)+\gamma^*)^++\frac{1-\|(\B{\Phi}_x^{\text{T}}(\B{\alpha}^*-\B{\beta}^*)+\V{1}\gamma^*)^+\|_1}{|\set{Y}|}
\end{align} 
for each $(x,y)\in\set{X}\times\set{Y}$. In what follows, we refer to such classification rules as \acp{LPC} for generating function $\Phi$, that is, classification rules $h^{\V{a},\V{b}}$ for $\V{a}, \V{b}\in\mathbb{R}^{m}$ and $\V{a}\preceq\V{b}$ given by \eqref{t-a,b} for $\B{\alpha}^*,\B{\beta}^*,\gamma^*$ solution of \eqref{learning-ineq}. 

The learning process of an \ac{LPC} consists on solving the convex optimization problem \eqref{learning-ineq}. The inputs of such learning process are expectation interval estimates given by $\V{a}\prec\V{b}$ or expectation point estimates given by $\V{a}=\V{b}$. Such estimates can be obtained by averaging the values that the generating function $\Phi$ takes over the training samples. 
Then, the prediction process with an \ac{LPC} for a specific $x\in\set{X}$ consists on randomly sample a label $y$ with probability given by \eqref{t-a,b} using $\B{\alpha}^*,\B{\beta}^*,\gamma^*$ obtained during learning.


Optimization problem \eqref{learning-ineq} is equivalent to a linear optimization problem with at most $|\set{X}|(2^{|\set{Y}|}-1)+2m$ constraints. Specifically,
$$\|(\B{\Phi}^{\text{T}}(\B{\alpha}-\B{\beta})+\V{1}\gamma)^+\|_{1,\infty}\leq 1\Leftrightarrow (\sum_{y\in\set{S}}\Phi(x,y)^{\text{T}})(\B{\alpha}-\B{\beta})+|\set{S}| \gamma \leq 1,\  \forall x\in\set{X}, \set{S}\subseteq\set{Y},\set{S}\neq\emptyset$$
because 
$$\|(\B{\Phi}^{\text{T}}(\B{\alpha}-\B{\beta})+\V{1}\gamma)^+\|_{1,\infty}\leq 1\Leftrightarrow \|\left(\B{\Phi}_{x}^{\text{T}}(\B{\alpha}-\B{\beta})+\V{1}\gamma\right)^+\|_1\leq 1,\  \forall x\in\set{X}$$
and 
$$\|\left(\B{\Phi}_{x}^{\text{T}}(\B{\alpha}-\B{\beta})+\V{1}\gamma\right)^+\|_1=\max_{\set{S}\subseteq \set{Y}}(\sum_{y\in\set{S}}\Phi(x,y)^{\text{T}})(\B{\alpha}-\B{\beta})+\sum_{y\in\set{S}}\gamma.$$

If case of using expectation point estimates, i.e., $\V{a}=\V{b}$, we can take the variables in \eqref{learning-ineq} to be $\gamma$ and $\B{\lambda}=\B{\alpha}-\B{\beta}\in\mathbb{R}^m$. In that case, \eqref{learning-ineq} is equivalent to
\begin{align}\label{learning-eq}\begin{array}{cc}\underset{\gamma\in\mathbb{R},\B{\lambda}\in\mathbb{R}^{m}}{\max}&\V{a}^{\text{T}}\B{\lambda}+\gamma\\
\mbox{s. t.}&\|(\B{\Phi}^{\text{T}}\B{\lambda}+\V{1}\gamma)^+\|_{1,\infty}\leq 1\end{array}\end{align}
that is an optimization problem with $m$ less dimensions and $2m$ less constraints than \eqref{learning-ineq}.

The learning process in \cite{AsiXinBeh:15} determines approximately minimax classification rules by addressing optimization such as that in \eqref{minimax-risk} for case $\V{a}=\V{b}$ using a stochastic gradient descent algorithm. Such approach is enabled by using the training samples' empirical distribution as surrogate for the features' marginal of distributions in the uncertainty set. The proposed learning process for \acp{LPC} described in Theorem~\ref{th1} does not rely on approximations and finds minimax classification rules by using linear optimization.




The following result shows that the usage of polyhedral uncertainty sets also allows to obtain performance guarantees (bounds for expected losses) by solving two linear optimization problems.

\begin{proposition}\label{prop}
Let  
\begin{align}\label{lower}\begin{array}{cccc}\kappa_\Phi^{\V{a},\V{b}}(q)&=\underset{\gamma\in\mathbb{R},\B{\alpha},\B{\beta}\in\mathbb{R}^{m}}{\max}&\V{a}^{\text{T}}\B{\alpha}-\V{b}^{\text{T}}\B{\beta}+\gamma\\&
\mbox{s. t.}&\B{\Phi}^{\text{T}}(\B{\alpha}-\B{\beta})+\V{1}\gamma\preceq \V{q}\\&
&\B{\alpha},\B{\beta}\succeq \V{0} \end{array}\end{align}
for a function $q:\set{X}\times\set{Y}\to\mathbb{R}$.
Then, for any $p\in\set{U}_{\Phi}^{\V{a},\V{b}}$ and $h\in\Delta(\set{X},\set{Y})$ \begin{align}\label{bounds}0\leq1+\kappa_\Phi^{\V{a},\V{b}}(-h)\leq\ell(h,p)\leq1-\kappa_\Phi^{\V{a},\V{b}}(h)\leq 1.\end{align}
In addition, $\ell(h,p)=1+\kappa_\Phi^{\V{a},\V{b}}(-h)$ (resp. $\ell(h,p)=1-\kappa_\Phi^{\V{a},\V{b}}(h)$) if $p$ minimizes (resp. maximizes) the expected loss of $h$ over distributions in $\set{U}_\Phi^{\V{a},\V{b}}$. 
\end{proposition}
\vspace{-0.3cm}
\begin{proof}
See Appendix~\ref{proof-prop}.
\end{proof}
\vspace{-0.3cm}
For an \ac{LPC} $h^{\V{a},\V{b}}$, the upper bound above is directly given by the learning phase, that is, $R_\Phi^{\V{a},\V{b}}$ given in \eqref{upper} equals $1-\kappa_\Phi^{\V{a},\V{b}}(h^{\V{a},\V{b}})$. On the other hand, the lower bound for $h^{\V{a},\V{b}}$ denoted as $L_\Phi^{\V{a},\V{b}}= 1+\kappa_\Phi^{\V{a},\V{b}}(-h^{\V{a},\V{b}})$ requires to solve an additional linear optimization problem.

Techniques based on f-divergences and Wasserstein distances in \cite{DucGlyNam:16,AbaMohKuh:15,ShaKuhMoh:17}
obtain analogous upper and lower bounds for the corresponding uncertainty sets. 
Note that the bounds for expected losses become risk's bounds if the actual distribution of features-label pairs belongs to the uncertainty set. Such case can be attained with a tunable confidence using uncertainty sets defined by Wasserstein distances as in \cite{AbaMohKuh:15,ShaKuhMoh:17} or using the proposed \acp{LPC} with expectation confidence intervals. However, the bounds are only  asymptotical risk's bounds using uncertainty sets defined by f-divergences as in \cite{DucGlyNam:16} or using the proposed \acp{LPC} with expectation point estimates.
\section{Generalization bounds}\label{sec-3}
In this section we develop finite-sample risk's bounds of \acp{LPC} with respect to the smallest worst-case risk for generating function $\Phi$. If the actual distribution of pairs features-label $p^*$ is contained in $\set{U}_\Phi^{\V{a},\V{b}}$, the minimax expected loss $R_\Phi^{\V{a},\V{b}}$ is the worst-case risk of $h^{\V{a},\V{b}}$ since $R(h^{\V{a},\V{b}})\leq R_\Phi^{\V{a},\V{b}}$ with equality if $p^*$ is a distribution in $\set{U}_\Phi^{\V{a},\V{b}}$ with smallest $(\infty,1)$-norm. 

The smallest worst-case risk of \acp{LPC} for generating function $\Phi$ is $R_\Phi^{\B{\tau}_\infty}$ with $\B{\tau}_\infty=\mathbb{E}
_{p^*}\{\Phi\}$ because 
$$p^*\in\set{U}_\Phi^{\V{a},\V{b}}\Rightarrow \set{U}_\Phi^{\B{\tau}_\infty}\subseteq\set{U}_\Phi^{\V{a},\V{b}}\Rightarrow R_\Phi^{\B{\tau}_\infty}\leq R_\Phi^{\V{a},\V{b}}.$$ Such smallest worst-case risk corresponds with \ac{LPC} $h^{\B{\tau}_\infty}$ that would require an infinite amount of training samples to exactly determine the expectation of generating function $\Phi$.

The following result bounds the excess risk of \acp{LPC} with respect to smallest worst-case risk, as well as the difference between the risk of \acp{LPC} and the corresponding minimax expected loss
\begin{theorem}\label{th-bounds}
Let $(x_{1},y_{1}),(x_{2},y_{2}),\ldots,(x_{n},y_{n})$ be $n$ independent samples following distribution $p^*$, $\Phi:\set{X}\times\set{Y}\to\mathbb{R}^m$ a generating function, $\delta\in(0,1)$ , $\B{\tau}_\infty=\mathbb{E}_{p^*}\{\Phi\}$, and
\begin{align}\label{interval} \B{\tau}_n=\frac{1}{n}\sum_{i=1}^n\Phi(x_{i},y_{i}),\  \V{a}_n=\B{\tau}_n-\V{s}\sqrt{\frac{1}{n}},\  \V{b}_n=\B{\tau}_n+\V{s}\sqrt{\frac{1}{n}}\end{align}
with 
$$\V{s}=\V{c}\sqrt{\frac{\log m+\log\frac{2}{\delta}}{2}}\mbox{, and } c_i=\max_{x\in\set{X},y\in\set{Y}} (\Phi(x,y))_i-\min_{x\in\set{X},y\in\set{Y}}(\Phi(x,y))_i\mbox{,  for $i=1,2,\ldots,m$}.$$
We have that
\begin{itemize}
\item[i)] With probability at least $1-\delta$,
\begin{align}
L_\Phi^{\V{a}_n,\V{b}_n}\leq R(h^{\V{a}_n,\V{b}_n})&\leq R_\Phi^{\V{a}_n,\V{b}_n}\leq R_\Phi^{\B{\tau}_\infty}+2M_\Phi\|\V{c}\|_2\sqrt{\frac{\log m+\log\frac{2}{\delta}}{2}}\frac{1}{\sqrt{n}}\label{bound1}\\
R(h^{\B{\tau}_n})&\leq R_\Phi^{\B{\tau}_n}+M_\Phi\|\V{c}\|_2\sqrt{\frac{\log m+\log\frac{2}{\delta}}{2}}\frac{1}{\sqrt{n}}\label{bound2}\\
R(h^{\B{\tau}_n})&\geq L_\Phi^{\B{\tau}_n}-M_\Phi\|\V{c}\|_2\sqrt{\frac{\log m+\log\frac{2}{\delta}}{2}}\frac{1}{\sqrt{n}}\label{bound3}\\
R(h^{\B{\tau}_n})&\leq R_\Phi^{\B{\tau}_\infty}+N_\Phi\|\V{c}\|_2\sqrt{\frac{\log m+\log\frac{2}{\delta}}{2}}\frac{1}{\sqrt{n}}\label{bound4}\end{align}
where
$M_\Phi=\max_{\B{\lambda}\in\Lambda}\|\B{\lambda}\|_2$, and $N_\Phi=\max_{\B{\lambda}_1,\B{\lambda}_2\in\Lambda}\|\B{\lambda}_1-\B{\lambda}_2\|_2$
for $$\Lambda=\{\B{\lambda}\in\mathbb{R}^m:\  \exists \V{a}\in\text{Conv}(\Phi(\set{X}\times\set{Y}))\mbox{ s.t. }  \gamma,\B{\lambda}\mbox{ is solution of \eqref{learning-eq} for some }\gamma\in\mathbb{R}\}.$$
\item[ii)] If \eqref{learning-eq} for $\V{a}=\B{\tau}_\infty$ has unique solution, then with probability $1$
$$R(h^{\V{a}_n,\V{b}_n})\underset{n\to\infty}{\to}R(h^{\B{\tau}_\infty})$$
$$R(h^{\B{\tau}_n})\underset{n\to\infty}{\to}R(h^{\B{\tau}_\infty})$$
\end{itemize}
\end{theorem}
\vspace{-0.3cm}
\begin{proof}
See Appendix~\ref{proof-bounds}.
\end{proof}
\vspace{-0.3cm}
Inequality \eqref{bound4} and third inequality in \eqref{bound1} bound the excess risk of \acp{LPC} with respect to the smallest worst-case risk $R_\Phi^{\B{\tau}_\infty}$; inequality \eqref{bound2} and second inequality in \eqref{bound1} bound the difference between the risk of \acp{LPC} and the corresponding minimax expected loss; and inequality \eqref{bound3} and first inequality in \eqref{bound1} bound the difference between the lower bound for the corresponding uncertainty set and the risk of \acp{LPC}. These bounds show differences that decrease with $n$ as $O(1/\sqrt{n})$ with proportionality constants that depend on the confidence $\delta$, and other parameters describing the complexity of generating function $\Phi$ such as its dimensionality $m$, the difference between its maximum and minimum values $\V{c}$, and bounds for the solutions of \eqref{learning-eq} for vectors $\V{a}$ in the convex hull of $\Phi(\set{X}\times\set{Y})$.

The vector $\V{s}$ above can result in over-pessimistic interval estimates $\V{a}_n$ and $\V{b}_n$ for the expectation of $\Phi$ since it is based on Hoeffding's inequality and the union bound \cite{BouLugMas:13} for the $m$ components of $\Phi$. In practice, \acp{LPC} can be developed by using tighter interval estimates for the expectation of $\Phi$. Such tighter intervals can be obtained for instance by using bootstrapping methods, the central limit theorem, and better estimates of sub-Gaussian parameters than $\V{c}$.


The generalization bounds for the excess risk provided in Theorem~3 of \cite{FarTse:16} and Theorems~2 and 3 of \cite{LeeRag:18} for \ac{RRM} with moments fits and Wasserstein distances, respectively, are analogous to those in inequality \eqref{bound4} and third inequality in \eqref{bound1} above. In particular, they also show risk's bounds with respect to the minimax risk corresponding to an infinite number of samples. The generalization bounds in Corollary~3.2 in \cite{NamDuc:17} and Theorem~2 of \cite{AbaMohKuh:15} for \ac{RRM} with f-divergences and Wasserstein distances, respectively, are analogous to those in inequality \eqref{bound2} and second inequality in \eqref{bound1}. In particular, they also show how the risk can be upper bounded (assymptotically in \cite{NamDuc:17} and inequality \eqref{bound2} or with certain confidence in \cite{AbaMohKuh:15} and second inequality in \eqref{bound1}) by the corresponding finite-sample minimax expected loss.
\section{Efficient implementation and choice of generating function}\label{sec-4}
The learning stage of \acp{LPC} entails to solve optimization problem \eqref{learning-ineq} using expectation interval estimates $\V{a}$ and $\V{b}$ or optimization problem \eqref{learning-eq} using expectation point estimates $\V{a}$. Training samples are used to obtain such estimates for the expectations of $\Phi$, and can be used also to select generating function $\Phi$ as described below. In the prediction stage, each $x\in\set{X}$ is classified as $y\in\set{Y}$ with probability $h^{\V{a},\V{b}}(x,y)$ given by \eqref{t-a,b} using generating function $\Phi$ and $\B{\alpha}^*, \B{\beta}^*, \B{\gamma}^*$ obtained in the learning stage. The upper bound for the risk of $h^{\V{a},\V{b}}$ is directly obtained from the learning stage as $R_\Phi^{\V{a},\V{b}}$ given by \eqref{upper}, while the lower bound for the risk of $h^{\V{a},\V{b}}$ is obtained solving an additional linear optimization problem as $L_\Phi^{\V{a},\V{b}}=1+\kappa_\Phi^{\V{a},\V{b}}(-h^{\V{a},\V{b}})$ for $\kappa_\Phi^{\V{a},\V{b}}(\cdot)$ given by \eqref{lower}.

The main complexity of \acp{LPC} lies in the possibly large number of constraints in the optimization problem solved for learning. As described in Section~\ref{sec-2}, optimization problems given in \eqref{learning-ineq}, \eqref{learning-eq}, and \eqref{lower} can have up to $|\set{X}|(2^{|\set{Y}|}-1)+2m$ linear constraints and $|\set{X}|$ is usually large. Such complexity can be controlled by i) using generating function $\Phi$ that takes a reduced number of values, and ii) approximately solving the optimization problems enforcing only a subset of constraints. Specifically, for i) if $\{\V{M}_{1},\V{M}_{2},\ldots,\V{M}_{r}\in\mathbb{R}^{|\set{Y}|\times m}\}=\{\B{\Phi}_x^{\text{T}}\in\mathbb{R}^{|\set{Y}|\times m}:\ x\in\set{X}\}$ optimization problems in \eqref{learning-ineq} and \eqref{learning-eq} for learning have up to $r(2^{|\set{Y}|}-1)+2m$ linear constraints, e.g., \eqref{learning-ineq} is equivalent to
\begin{align}\label{learning-eff}\begin{array}{cc}\underset{\gamma\in\mathbb{R},\B{\alpha},\B{\beta}\in\mathbb{R}^{m}}{\max}&\V{a}^{\text{T}}\B{\alpha}-\V{b}^{\text{T}}\B{\beta}+\gamma\\
\mbox{s. t.}&\|\left(\V{M}_{i}(\B{\alpha}-\B{\beta})+\V{1}\gamma\right)^+\|_1\leq 1,\  i=1,2,\ldots,r\\
&\B{\alpha},\B{\beta}\succeq \V{0} \end{array}\end{align}
For ii), if $\{x_1,x_2,\ldots,x_N\}$ is a subset of $\set{X}$ (e.g., features obtained in training), \eqref{learning-ineq}, \eqref{learning-eq}, and \eqref{lower} can be approximated by optimization problems with up to $N(2^{|\set{Y}|}-1)+2m$ linear constraints, e.g., \eqref{learning-ineq} can be approximated by
\begin{align*}\begin{array}{cc}\underset{\gamma\in\mathbb{R},\B{\alpha},\B{\beta}\in\mathbb{R}^{m}}{\max}&\V{a}^{\text{T}}\B{\alpha}-\V{b}^{\text{T}}\B{\beta}+\gamma\\
\mbox{s. t.}&\|\left(\B{\Phi}_{x_i}^{\text{T}}(\B{\alpha}-\B{\beta})+\V{1}\gamma\right)^+\|_1\leq 1,\  i=1,2,\ldots,N\\
&\B{\alpha},\B{\beta}\succeq \V{0} \end{array}\end{align*}




The generating function $\Phi$ plays an analogous role to that of predicates in \cite{VapIzm:18}, which represent the contribution to the training process of a so-called Intelligent Teacher. Such type of functions are used also in other methods for \ac{RRM} \cite{FarTse:16,AsiXinBeh:15} and are usually obtained from certain moments of the features; while, as pointed out in \cite{VapIzm:18}, improved performance can be obtained by more elaborated functions possibly defined algorithmically.

The generating function $\Phi$ used by an \ac{LPC} has to be highly discriminative for classification ($R_\Phi^{\B{\tau}_\infty}\approx R_{\text{Bayes}}$) and, at the same time, simple enough to enable efficient learning (reduced dimensionality $m$ and range of values $r$). Ideal generating function $\Phi$ would be that given by the Bayes rule $\Phi=h_{\text{Bayes}}:\set{X}\times\set{Y}\to\{0,1\}$ because in that case  $R_\Phi^{\B{\tau}_\infty}= R_{\text{Bayes}}$, $m=1$, and $r=|\set{Y}|$. Other generating functions that achieve $R_\Phi^{\B{\tau}_\infty}= R_{\text{Bayes}}$ are those such that $h_{\text{Bayes}}$ is in the linear span of $\Phi$, because $$\ell(\B{\theta}^{\text{T}}\Phi,p)=1-\B{\theta}^{\text{T}}\B{\Phi}\V{p}=1-\B{\theta}^{\text{T}}\B{\tau}_\infty=\ell(\B{\theta}^{\text{T}}\Phi,p^*) \mbox{ for }p\in\set{U}_\Phi^{\B{\tau}_\infty}\mbox{ and }\B{\theta}\in\mathbb{R}^m.$$ 
In the numerical results of next section we use a simple generating function $\Phi$ given by $k$ classifiers $h_1,h_2,\ldots,h_k$ as $\Phi:
\set{X}\times\set{Y}\to\{0,1\}^{|\set{Y}|^{k+1}}$ where for each $(x,y)\in\set{X}\times\set{Y}$, $\Phi(x,y)$ is a vector of size $m=|\set{Y}|^{k+1}$ with $m-1$ zeros and a $1$ at the component corresponding to the $(k+1)$-tuple $(y,h_1(x),h_2(x)\ldots,h_k(x))$, that is, for $j=1,2,\ldots,m$
\begin{align}\label{generating}(\Phi(x,y))_{j}=\left\{\begin{array}{cc}1&\mbox{if }\  \text{Ind}(y,h_1(x),h_2(x)\ldots,h_k(x))=j\\0&\mbox{otherwise}\end{array}\right.\end{align}
where $\text{Ind}(y_0,y_1,\ldots,y_k)$ assigns the integer index of tuple $(y_0,y_1,\ldots,y_k)$ for a chosen order such as lexicographic. 
For this function $\Phi$ we have that $r=|\set{Y}|^{k}$ and the matrix defining the constraints in \eqref{learning-eff} corresponding with $k$-tuple $(y_1,y_2,\ldots,y_k)$ has $(i,j)$-th component for $i\in\set{Y}$ and $j=1,2,\ldots,m$
$$\left(\V{M}_{\text{Ind}(y_1,y_2,\ldots,y_k)}\right)_{i,j}=\left\{\begin{array}{cc}1&\mbox{if }\  \text{Ind}(i,y_1,y_2\ldots,y_k)=j\\0&\mbox{otherwise}\end{array}\right..$$
Note that such matrices are highly sparse since each row has only one non-zero component, so high-efficient optimization methods can be exploited for \eqref{learning-ineq} and \eqref{learning-eq}. On the other hand, for this type of generating functions the constant $\|\V{c}\|_2$
in Theorem~\ref{th-bounds} above becomes $\sqrt{m}=|\set{Y}|^{\frac{k+1}{2}}$, and the dimensionality and number of linear constraints in \eqref{learning-eff} become $O(|\set{Y}|^{k+1})$. Therefore, this type of generating functions requires to use a reduced number of classifiers $k$. In the next section we use $k=3$ so the learning process entails to solve linear optimization problems with up to $|\set{Y}|^3(2^{|\set{Y}|}-1)+2|\set{Y}|^4$ linear constraints and $2|\set{Y}|^4+1$ dimensions. In the next section, the expectations of the proposed generating function are estimated from training data using stratified $10$-fold cross-validation. Specifically, each validation sample provides an evaluation of the generating function and the final estimate is obtained by averaging the estimates corresponding with each data partition. 




\section{Experimental results}\label{sec-5}
In this section we show numerical results for \acp{LPC} using synthetic data and UCI datasets. The first set of results shows the suitability of the upper and lower bounds $R_\Phi^{\V{a},\V{b}}$ and $L_\Phi^{\V{a},\V{b}}$ for \acp{LPC}, while the second set of results compares the classification error of \acp{LPC} with respect to state-of-the-art techniques. 


In the first set of experimental results, we use synthetic data for classification with $4$-dimensional features and 3 classes. Specifically, the features for each class are obtained as random samples from a mixture of two Gaussians with weights $0.5$ and covariances $0.7^2 I$, the means of the Gaussians are $(1,1,1,1)$ and $(3,3,3,3)$ for $y=1$,  $(1,2,1,2)$ and $(4,3,4,3)$ for $y=2$, and $(2,2,2,2)$ and $(4,4,4,4)$ for $y=3$. The \ac{LPC} in this set of results uses generating function in \eqref{generating} with $k=3$ and $h_1$, $h_2$, and $h_3$ given by \ac{NN} algorithms with $3$, $5$, and $7$ neighbors.\begin{wrapfigure}{r}{0.5\textwidth}
\psfrag{Y}[l][t][0.7]{\hspace{-7mm}Risk}
\psfrag{X}[l][b][0.7]{\hspace{-10mm}Training size n}
\psfrag{A123456789123456}[l][][0.5]{\hspace{-8.7mm}\ac{LPC} Risk}
\psfrag{B}[l][][0.5]{\hspace{0.5mm}Upper bound}
\psfrag{C}[l][][0.5]{\hspace{0.5mm}Lower bound}
\psfrag{D}[l][][0.5]{\hspace{0.5mm}Bayes Risk}
\psfrag{50}[l][][0.5]{\hspace{-2mm}$50$}
\psfrag{0.2}[l][][0.5]{\hspace{-3mm}$0.2$}
\psfrag{0.3}[l][][0.5]{\hspace{-3mm}$0.3$}
\psfrag{0.4}[l][][0.5]{\hspace{-3mm}$0.4$}
\psfrag{0.5}[l][][0.5]{\hspace{-3mm}$0.5$}
\psfrag{0.6}[l][][0.5]{\hspace{-3mm}$0.6$}
\psfrag{100}[l][][0.5]{\hspace{-2mm}$100$}
\psfrag{1000}[l][][0.5]{\hspace{-4mm}$1000$}
\psfrag{5000}[l][][0.5]{\hspace{-4mm}$5000$}
\psfrag{2}[l][b][0.7]{\hspace{-5.5mm}Diabetes}
\psfrag{3}[l][b][0.7]{\hspace{-5.5mm}German}
\psfrag{4}[l][b][0.7]{\hspace{-4mm}Heart}
	\includegraphics[width=0.5\textwidth]{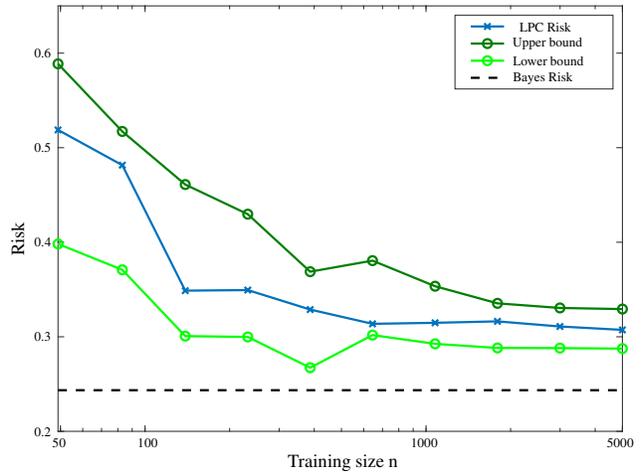}
	\caption{Upper and lower \ac{LPC} risk bounds.\label{fig_bounds}}
\end{wrapfigure} Figure~\ref{fig_bounds} shows the risk of an \ac{LPC} that uses $\V{a}_n$ and $\V{b}_n$ given by \eqref{interval} with $\V{s}=0.25\V{1}$. For each training size, one instantiation of training samples is used for training and \ac{LPC}'s risk is estimated using $10,000$ test samples. It can be observed from the figure that the lower and upper bounds can offer accurate estimates for the risk without using test samples.

In the second set of experimental results, we use $10$ data sets from the UCI repository (first column of Table~\ref{table:results}). \acp{LPC} are compared with $8$ classifiers: \ac{DT}, \ac{QDA}, \ac{NN}, \ac{SVM}, \ac{RF}, \ac{LR}, \ac{LUSI}, and \ac{ACSC}. The first 6 classifiers were implemented using scikit-learn package with the default parameters, \ac{LUSI} was implemented as in \cite{VapIzm:18} using parameters $\gamma=0.1$ and $\delta=1.0$, and \ac{ACSC} was implemented as in \cite{AsiXinBeh:15} using parameter $\gamma_t=0.1/\sqrt{t}$. Two versions of the proposed classifiers LPC1 and LPC2 were implemented using different classifiers $h_1$, $h_2$, and $h_3$ to define $\Phi$ in \eqref{generating}; LPC1 uses \ac{DT}, \ac{QDA}, and $5$\ac{NN} while LPC2 uses  \ac{SVM}, \ac{RF}, and \ac{LR}. The errors in Table~\ref{table:results} have been estimated using paired and stratified $10$-fold cross validation. It can be observed from the table that performance of \acp{LPC} is competitive with state-of-the-art techniques.



\begin{table}
\caption{Classification error of \acp{LPC} in comparison with state-of-the-art techniques.}
\label{table:results}
\centering
\begin{tabular}{lcccccccccc}
\toprule
data set& LPC1& LPC2& QDA& DT& KNN& SVM& RF& LR& ACSC& LUSI\\
\hline
mammog.& .21& .19& .20& .24& .22& .18& .21& .17& .22& .18\\
vehicle& .16& .21& .15& .28& .29& .21& .25& .21& .37& .22\\
glass& .42& .38& .49& .39& .35& .35& .40& .40& .41&.35\\
haberman& .26& .27& .24& .39& .30& .26& .35& .26& .28& .26\\
column 3C& .15& .16& .16& .20& .21& .15& .17& .15& .27& .17\\
indian liver& .29& .28& .45& .35& .34& .29& .30& .28& .33& .27\\
diabetes& .26& .23& .26& .30& .26& .24& .26& .23& .29& .23\\
adult& .15& .15& .20& .18& .17& .15&.15& .18& .20&.15\\
credit& .15& .17& .22& .22& .14& .16& .17& .15& .22& .16\\
satellite& .12& .12& .16& .17& .12& .12&.11& .18& .18& .11\\
\bottomrule
\end{tabular}
\end{table}

\section{Conclusion}
The proposed \acp{LPC} consider unconstrained classification rules, optimize the classification 0-1 loss, and provide performance guarantees during learning. We present \acp{LPC}' finite-sample generalization bounds, and describe practical and efficient implementations. This paper shows that supervised classification does not require to select from the outset a family of classification rules or surrogate losses with favorable tractability properties. Differently from conventional techniques, the inductive bias exploited by \acp{LPC} comes from a chosen generating function that represents the classification-discriminative characteristics of examples. Learning with \acp{LPC} is achieved without further design choices by linear optimization problems given by expectation estimates obtained from training data. Finally, we propose a simple choice for generating function that results in \acp{LPC} achieving classification errors competitive with state-of-the-art techniques.

\newpage
\bibliographystyle{unsrt}
\bibliography{bib-santi}
\newpage
\section{Appendices}

\subsection{Auxiliary lemmas}
The proofs of Theorem~\ref{th1} and Proposition~\ref{prop} require the lemmas provided below.

\begin{lemma}\label{lemma-dual}
The norms $\|\cdot\|_{\infty,1}$ and $\|\cdot\|_{1,\infty}$ are dual. 
\end{lemma}
\begin{proof}

The dual norm of $\|\cdot\|_{\infty,1}$ assigns each $\V{w}\in\mathbb{R}^{|\set{I}||\set{J}|}$, the real number
$$\sup_{\V{v}:\  \|\V{v}\|_{\infty,1}\leq 1}\V{w}^{\text{T}}\V{v}$$
We have that for $\V{v}$ with $\|\V{v}\|_{\infty,1}\leq 1$
\begin{align*}\V{w}^{\text{T}}\V{v}&=\sum_{i\in\set{I}}\sum_{j\in\set{J}}w_{(i,j)}v_{(i,j)}\leq\sum_{i\in\set{I}}\sum_{j\in\set{J}}|w_{(i,j)}||v_{(i,j)}|\\
&\leq \sum_{i\in\set{I}}\left(\max_j|v_{(i,j)}|\right)\sum_{j\in\set{J}}|w_{(i,j)}|\leq\max_{i\in\set{I}}\sum_{j\in\set{J}}|w_{(i,j)}|\sum_{i\in\set{I}}\left(\max_j|v_{(i,j)}|\right)\\&=\|\V{w}\|_{1,\infty}\|\V{v}\|_{\infty,1}\leq\|\V{w}\|_{1,\infty}\end{align*}
So, to prove the result we just need to find a vector $\V{u}$ such that $\|\V{u}\|_{\infty,1}\leq 1$ and $\V{w}^{\text{T}}\V{u}=\|\V{w}\|_{1,\infty}$. Let $\tilde{i}\in\arg\max_{i\in\set{I}}\sum_{j\in\set{J}}|w_{(i,j)}|$, then $\V{u}$ given by 
$$u_{(i,j)}=\left\{\begin{array}{cc}1&\mbox{ if }i=\tilde{i} \mbox{ and } w_{(i,j)}\geq 0\\
-1&\mbox{ if }i=\tilde{i} \mbox{ and } w_{(i,j)}< 0\\
0&\mbox{ otherwise }\end{array}\right.$$
satisfies $\|\V{u}\|_{\infty,1}\leq 1$ and $\V{w}^{\text{T}}\V{u}=\|\V{w}\|_{1,\infty}$.

\end{proof}

\begin{lemma}\label{lemma-conjugate}
Let $\V{u}\in\mathbb{R}^{|\set{I}||\set{J}|}$, and $f_1$ and $f_2$ be the functions $f_1(\V{v})=\|\V{v}\|_{\infty,1}+I^+(\V{v})$ and $f_2(\V{v})=\V{v}^{\text{T}}\V{u}+I^+(\V{v})$ for $\V{v}\in\mathbb{R}^{|\set{I}||\set{J}|}$, where
$$I^+(\V{v})=\left\{\begin{array}{cc}0 &\mbox{if}\  \V{v}\succeq\V{0} \\\infty&\mbox{otherwise}\end{array}\right.$$
Then, their conjugate functions are
$$f_1^*(\V{w})=\left\{\begin{array}{cc}0 &\mbox{if}\  \|\V{w}^+\|_{1,\infty}\leq 1 \\\infty&\mbox{otherwise}\end{array}\right.$$
$$f_2^*(\V{w})=\left\{\begin{array}{cc}0 &\mbox{if}\  \V{w}\preceq\V{u}  \\\infty&\mbox{otherwise}\end{array}\right.$$

\end{lemma}

\begin{proof}
	By definition of conjugate function we have
$$f_1^*(\V{w}) =\sup_{\V{v}} (\V{w}^{\text{T}}\V{v} - \|\V{v}\|_{\infty,1} - I^+(\V{v})) = \sup_{\V{v} \succeq 0} (\V{w}^{\text{T}}\V{v} -  \|\V{v}\|_{\infty,1}).$$
\begin{itemize}
\item If $\|\V{w}^+\|_{1,\infty} \leq 1$, for each $\V{v} \succeq \V{0}$, $\V{v} \neq \V{0}$ we have
	$$\V{w}^{\text{T}}\V{v} \leq (\V{w}^+)^{\text{T}}\V{v} = \|\V{v}\|_{\infty,1} \left((\V{w}^+)^{\text{T}}\frac{\V{v}}{\|\V{v}\|_{\infty,1}}\right)$$ 
	and by definition of dual norm we get
	$$\V{w}^{\text{T}}\V{v} \leq \|\V{v}\|_{\infty,1} \|\V{w}^+\|_{1,\infty} \leq \|\V{v}\|_{\infty,1}$$
	which implies
	$$ \V{w}^{\text{T}}\V{v} - \|\V{v}\|_{\infty,1}\leq 0.$$
	Moreover, $\V{w}^{\text{T}}\V{0} - \|\V{0}\|_{\infty,1}= 0$, so we have that $f_1^*(\V{w}) = 0$. 
\item If $\|\V{w}^+\|_{1,\infty} > 1$, by definition of dual norm and using Lemma~\ref{lemma-dual} there exists $\V{u}$ such that $(\V{w}^+)^{\text{T}}\V{u} > 1$ and $\|\V{u}\|_{\infty,1} \leq 1$. Define $\tilde{\V{u}}$ as
	$$
	\tilde{u}_{(i,j)} =\left\{\begin{array}{cc}
	u_{(i,j)} & \text{ if } u_{(i,j)} \geq 0 \text{ and } w_{(i,j)} \geq 0 \\
	0 & \text{ if } u_{(i,j)} < 0 \text{ or } w_{(i,j)} < 0
	\end{array}\right.
	$$
	By definition of $\tilde{\V{u}}$ and $\|\cdot\|_{\infty,1}$ we have
	$$\|\tilde{\V{u}}\|_{\infty,1} \leq \|\V{u}\|_{\infty,1} \leq 1 $$
	and
	$$\V{w}^{\text{T}}\tilde{\V{u}} = (\V{w}^+)^{\text{T}}\tilde{\V{u}}\geq(\V{w}^+)^{\text{T}}\V{u} > 1. $$
	Now let $t > 0$ and take $\V{v} = t \tilde{\V{u}} \succeq 0$, then we have
	$$\V{w}^{\text{T}}\V{v} - \|\V{v}\|_{\infty,1} = t \left ( \V{w}^{\text{T}}\tilde{\V{u}} - \|\tilde{\V{u}}\|_{\infty,1} \right ) $$
	which tends to infinity as $t \to + \infty$ because $\V{w}^{\text{T}}\tilde{\V{u}} - \|\tilde{\V{u}}\|_{\infty,1} > 0$, so we have that $f_1^*(\V{w}) = + \infty$.
	\end{itemize}
	Finally, the expression for $f_2^*$ is straightforward since	$$f_2^*(\V{w})=\sup_{\V{v}\succeq\V{0}}((\V{w}-\V{u})^{\text{T}}\V{v}).$$
\end{proof}

\subsection{Proof of Theorem~\ref{th1}}\label{proof-th1}

Let 
$$\widetilde{\set{U}}=\{p:\set{X}\times\set{Y}\to\mathbb{R}\mbox{ s.t. }  \V{p}\succeq \V{0},\  \|\V{p}\|_{1,\infty}\leq 1\}$$
$$\widetilde{\ell}(h,p)=1-\V{a}^{\text{T}}\B{\alpha}^*+\V{b}^{\text{T}}\B{\beta}^*-\gamma^*+\V{p}^{\text{T}}(\B{\Phi}^{\text{T}}(\B{\alpha}^*-\B{\beta}^*)+\V{1}\gamma^*-\V{h})$$
in the first step of the proof we show that $h^*$ satisfying \eqref{robust-act} is a solution of optimization problem $\min_{h\in\Delta(X,Y)}\max_{p\in\widetilde{\set{U}}}\widetilde{\ell}(h,p)$, and in the second step of the proof we show that a solution of $\min_{h\in\Delta(X,Y)}\max_{p\in\widetilde{\set{U}}}\widetilde{\ell}(h,p)$ is also a solution of $\min_{h\in\Delta(X,Y)}\max_{p\in\set{U}_{\Phi}^{\V{a},\V{b}}}\ell(h,p)$.

For the first step, note that
$$\widetilde{\ell}(h,p)=1-\V{a}^{\text{T}}\B{\alpha}^*+\V{b}^{\text{T}}\B{\beta}^*-\gamma^*+\sum_{x\in\set{X}}\V{p}_x^{\text{T}}\left(\B{\Phi}_x^{\text{T}}(\B{\alpha}^*-\B{\beta}^*)+\V{1}\gamma^*-\V{h}_x\right).$$


Then, optimization problem $\min_{h\in\Delta(X,Y)}\max_{p\in\widetilde{\set{U}}}\widetilde{\ell}(h,p)$ is equivalent to
$$\begin{array}{ccc}\min &\max &  \sum_{x\in\set{X}}\V{p}_x^{\text{T}}\left(\B{\Phi}_x^{\text{T}}(\B{\alpha}^*-\B{\beta}^*)+\V{1}\gamma^*-\V{h}_x\right)\\
h_x\in\Delta(\set{Y})\  \forall x\in\set{X}&\V{p}_x\succeq \V{0}, \|\V{p}_x\|_1\leq 1 \forall x\in\set{X}&\end{array}$$



that is separable and has solution given by
$$\begin{array}{cccc}h_x^*\in&\arg\min &\max &\V{p}_x^{\text{T}}\left(\B{\Phi}_x^{\text{T}}(\B{\alpha}^*-\B{\beta}^*)+\V{1}\gamma^*-\V{h}_x\right)\\&
h_x\in\Delta(\set{Y})&\V{p}_x\succeq \V{0}, \|\V{p}_x\|_1\leq 1&\end{array}$$
for any $x\in\set{X}$. The inner maximization above is given in closed-form by
$$\begin{array}{ccc}\max &\V{p}_x^{\text{T}}\left(\B{\Phi}_x^{\text{T}}(\B{\alpha}^*-\B{\beta}^*)+\V{1}\gamma^*-\V{h}_x\right)&=\|\left(\B{\Phi}_x^{\text{T}}(\B{\alpha}^*-\B{\beta}^*)+\V{1}\gamma^*-\V{h}_x\right)^+\|_\infty\geq 0\\
\V{p}_x\succeq \V{0}, \|\V{p}_x\|_1\leq 1&&\end{array}$$
that takes its minimum value $0$ for any $\V{h}_x^*\succeq \B{\Phi}_x^{\text{T}}(\B{\alpha}^*-\B{\beta}^*)+\V{1}\gamma^*$.



For the second step, if $h^*$ is a solution of $\min_{h\in\Delta(X,Y)}\max_{p\in\widetilde{\set{U}}}\widetilde{\ell}(h,p)$ we have that
\begin{align}\label{ineqs}\min_{h\in\Delta(X,Y)}\max_{p\in\widetilde{\set{U}}}\widetilde{\ell}(h,p)=\max_{p\in\widetilde{\set{U}}}\widetilde{\ell}(h^*,p)\geq\max_{p\in\set{U}_\Phi^{\V{a},\V{b}}}\ell(h^*,p)\geq\min_{h\in\Delta(X,Y)}\max_{p\in\set{U}_\Phi^{\V{a},\V{b}}}\ell(h,p)\end{align}
where the first inequality is due to the fact that $\set{U}_\Phi^{\V{a},\V{b}}\subset\widetilde{\set{U}}$ and $\widetilde{\ell}(h,p)\geq\ell(h,p)$ for 
$p\in\set{U}_\Phi^{\V{a},\V{b}}$ because $\V{p}^{\text{T}}\B{\Phi}^{\text{T}}(\B{\alpha}^*-\B{\beta}^*)\geq \V{a}^{\text{T}}\B{\alpha}^*-\V{b}^{\text{T}}\B{\beta}^*$ by definition of $\set{U}_\Phi^{\V{a},\V{b}}$ and since $\B{\alpha}^*,\B{\beta}^*\succeq \V{0}$.

Since $\ell(h,p)$ is bounded, $\set{X}\times\set{Y}$ is finite, and $\set{U}_\Phi^{\V{a},\V{b}}$ and $\Delta(X,Y)$ 
are closed and convex, the min and the max in $R_\Phi^{\V{a},\V{b}}=\min_{h\in\Delta(X,Y)}\max_{p\in\set{U}_\Phi^{\V{a},\V{b}}}\ell(h,p)$ can be interchanged (see e.g., Th. 5.1. in \cite{GruDaw:04}) and we have that  $R_\Phi^{\V{a},\V{b}}=\max_{p\in\set{U}_\Phi^{\V{a},\V{b}}}\min_{h\in\Delta(X,Y)}\ell(h,p)$. In addition,
$$\min_{h\in\Delta(X,Y)}\ell(h,p)=\min_{h\in\Delta(X,Y)}1-\V{p}^{\text{T}}\V{h}=1-\|\V{p}\|_{\infty,1}$$
because the optimization problem above is separable for $x\in\set{X}$ and
\begin{align}\label{entropy}\max_{h_x\in\Delta(\set{Y})}\V{p}_x^{\text{T}}\V{h}_x=\|\V{p}_x\|_\infty.\end{align}
Then $1-R_\Phi^{\V{a},\V{b}}=\min_{P\in\set{U}_\Phi^{\V{a},\V{b}}}\|\V{p}\|_{\infty,1}$ that can be written as
\begin{align}\begin{array}{cc}\min&\|\V{p}\|_{\infty,1}+I^+(\V{p})\\
\mbox{s. t.} &\V{p}^{\text{T}}\V{1}=1\\&
\V{a}\preceq\B{\Phi}\V{p}\preceq\V{b}\end{array}\label{opt}\end{align}
where 
$$I^+(\V{p})=\left\{\begin{array}{cc}0 &\mbox{if}\  \V{p}\succeq \V{0}\\\infty&\mbox{otherwise}\end{array}\right.$$

The Lagrange dual of the optimization problem \eqref{opt} is
\begin{align}\label{dual}\begin{array}{cc}\max&\V{a}^{\text{T}}\B{\alpha}-\V{b}^{\text{T}}\B{\beta}+\gamma-f^*\left(\B{\Phi}^{\text{T}}(\B{\alpha}-\B{\beta})+\V{1}\gamma\right)\\\gamma\in\mathbb{R},\B{\alpha},\B{\beta}\in\mathbb{R}^{m}&\\\mbox{s.t.}&
\B{\alpha}\succeq \V{0}, \B{\beta}\succeq \V{0}\end{array}\end{align}
where $f^*$ is the conjugate function of $f(\V{p})=\|\V{p}\|_{\infty,1}+I^+(\V{p})$ (see e.g., section 5.1.6 in \cite{BoyVan:04}). Then, optimization problem \eqref{dual} becomes \eqref{learning-ineq} using the Lemma~\ref{lemma-conjugate} above.

Strong duality holds between optimization problems \eqref{opt} and \eqref{learning-ineq} since constraints in \eqref{opt} are affine. Then, if $\B{\alpha}^*,\B{\beta}^*,\gamma^*$ is a solution of \eqref{learning-ineq} we have that $1-R_\Phi^{\V{a},\V{b}}$ is equal to
\begin{align}\label{opt2}\min_P\|\V{p}\|_{\infty,1}+I^+(\V{p})-(\V{p}^{\text{T}}\B{\Phi}^{\text{T}}-\V{a}^{\text{T}})\B{\alpha}^*+(\V{p}^{\text{T}}\B{\Phi}^{\text{T}}-\V{b}^{\text{T}})\B{\beta}^*-(\V{p}^{\text{T}}\V{1}-1)\gamma^*\end{align} that equals
$$\min_{P\in\widetilde{\set{U}}}\|\V{p}\|_{\infty,1}+\V{a}^{\text{T}}\B{\alpha}^*-\V{b}^{\text{T}}\B{\beta}^*+\gamma^*-\V{p}^{\text{T}}\left(\B{\Phi}^{\text{T}}(\B{\alpha}^*-\B{\beta}^*)+\V{1}\gamma^*\right)$$
since a solution of the primal problem \eqref{opt} belongs to $\widetilde{U}$ and is also a solution of \eqref{opt2}. Therefore, 
\begin{align*}R_\Phi^{\V{a},\V{b}}&=\max_{p\in\widetilde{\set{U}}}\min_{h\in\Delta(\set{X},\set{Y})}\ell(h,p)-\V{a}^{\text{T}}\B{\alpha}^*+\V{b}^{\text{T}}\B{\beta}^*-\gamma^*+\V{p}^{\text{T}}\left(\B{\Phi}^{\text{T}}(\B{\alpha}^*-\B{\beta}^*)+\V{1}\gamma^*\right)\\&=\max_{p\in\widetilde{\set{U}}}\min_{h\in\Delta(X,Y)}\widetilde{\ell}(h,p)=\min_{h\in\Delta(X,Y)}\max_{p\in\widetilde{\set{U}}}\widetilde{\ell}(h,p)\end{align*}
where the last equality is due to the fact that $\widetilde{\ell}(h,p)$ is bounded, $\set{X}\times\set{Y}$ is finite, and $\widetilde{\set{U}}$ and $\Delta(X,Y)$ are closed and convex. Then, inequalities in \eqref{ineqs} are in fact equalities and $h^*$ is solution of $\min_{h\in\Delta(X,Y)}\max_{p\in\set{U}_\Phi^{\V{a},\V{b}}}\ell(h,p)$.

\subsection{Proof of Proposition~\ref{prop}}\label{proof-prop}
The result is a direct consequence of the fact that for any $p\in\set{U}_{\Phi}^{\V{a},\V{b}}$
$$\min_{\widetilde{p}\in\set{U}_\Phi^{\V{a},\V{b}}}\ell(h,\widetilde{p})\leq\ell(h,p)\leq\max_{\widetilde{p}\in\set{U}_\Phi^{\V{a},\V{b}}}\ell(h,\widetilde{p})$$
and 
$$\min_{\widetilde{p}\in\set{U}_\Phi^{\V{a},\V{b}}}\ell(h,\widetilde{p})=1+\min_{\widetilde{p}\in\set{U}_\Phi^{\V{a},\V{b}}}\V{\widetilde{p}}^{\text{T}}(-\V{h})$$
$$\max_{\widetilde{p}\in\set{U}_\Phi^{\V{a},\V{b}}}\ell(h,\widetilde{p})=1-\min_{\widetilde{p}\in\set{U}_\Phi^{\V{a},\V{b}}}\V{\widetilde{p}}^{\text{T}}\V{h}.$$
The expression for $\kappa_\Phi^{\V{a},\V{b}}(q)$ in \eqref{lower} is obtained since
\begin{align}\label{opt_kappa}\begin{array}{ccc}\min_{\widetilde{p}\in\set{U}_\Phi^{\V{a},\V{b}}}\V{\widetilde{p}}^{\text{T}}\V{q}=&\min&\V{\widetilde{p}}^{\text{T}}\V{q}+I^+(\V{\widetilde{p}})\\&
\mbox{s. t.} &\V{\widetilde{p}}^{\text{T}}\V{1}=1\\&&
\V{a}\preceq\B{\Phi}\V{\widetilde{p}}\preceq\V{b}\end{array}\end{align}
where 
$$I^+(\V{\widetilde{p}})=\left\{\begin{array}{cc}0 &\mbox{if}\  \V{\widetilde{p}}\succeq \V{0}\\\infty&\mbox{otherwise}\end{array}\right.$$
Then, the Lagrange dual of the optimization problem \eqref{opt_kappa} is
\begin{align}\begin{array}{cc}\max&\V{a}^{\text{T}}\B{\alpha}-\V{b}^{\text{T}}\B{\beta}+\gamma-f^*\left(\B{\Phi}^{\text{T}}(\B{\alpha}-\B{\beta})+\V{1}\gamma\right)\\\gamma\in\mathbb{R},\B{\alpha},\B{\beta}\in\mathbb{R}^{m}&\\\mbox{s.t.}&
\B{\alpha}\succeq \V{0}, \B{\beta}\succeq \V{0}\end{array}\end{align}
where $f^*$ is the conjugate function of $f(\V{\widetilde{p}})=\V{\widetilde{p}}^{\text{T}}\V{q}+I^+(\V{\widetilde{p}})$ 
that leads to \eqref{lower} using Lemma~\ref{lemma-conjugate}.

\subsection{Proof of Theorem~\ref{th-bounds}}\label{proof-bounds}

1. Using Hoeffding's inequality \cite{BouLugMas:13} we have that for $i=1,2,\ldots,m$

$$\mathbb{P}\left\{|\tau_{\infty,i}-\tau_{n,i}|\leq c_i\sqrt{\frac{\log m+\log\frac{2}{\delta}}{2n}}\right\}\geq1-2\exp\left\{-\log m-\log\frac{2}{\delta}\right\}= 1-\frac{\delta}{m}$$
so that, using the union bound we have that 
\begin{align*}\mathbb{P}\Bigg\{|\tau_{\infty,i}-\tau_{n,i}|\leq c_i\sqrt{\frac{\log m+\log\frac{2}{\delta}}{2n}},\  &i=1,2,\ldots,m\Bigg\}
\\
&\geq 1-m+\sum_{i=1}^m\mathbb{P}\left\{|\tau_{\infty,i}-\tau_{n,i}|\leq c_i\sqrt{\frac{\log m+\log\frac{2}{\delta}}{2n}}\right\}\\&\geq 1-\delta.
\end{align*}
Hence, $p^*\in\set{U}_{\Phi}^{\V{a}_n,\V{b}_n}$ and 
$$\|\B{\tau}_{\infty}-\B{\tau}_{n}\|_2\leq \|\V{c}\|_2\sqrt{\frac{\log m+\log\frac{2}{\delta}}{2n}}$$
with probability at least $1-\delta$.


For the first result, we have that $R(h^{\V{a}_n,\V{b}_n})\leq R_\Phi^{\V{a}_n,\V{b}_n}$ with probability at least $1-\delta$ since $p^*\in\set{U}_{\Phi}^{\V{a}_n,\V{b}_n}$ with probability at least $1-\delta$. 
Let $\B{\lambda}^*,\gamma^*$ be a solution of \eqref{learning-eq} for $\V{a}=\B{\tau}^*$; $\left[(\B{\lambda}^*)^+,(-\B{\lambda}^*)^+,\gamma^*\right]$ is a feasible point of \eqref{learning-ineq} because $\B{\lambda}^*=(\B{\lambda}^*)^+-(-\B{\lambda}^*)^+$ and $\B{\lambda}^*,\gamma^*$ is a feasible point of \eqref{learning-eq}. Hence
$$R_\Phi^{\V{a}_n,\V{b}_n}\leq 1-\V{a}_n^{\text{T}}(\B{\lambda}^*)^++\V{b}_n^{\text{T}}(-\B{\lambda}^*)^+-\gamma^*=R_\Phi^{\B{\tau}_{\infty}}+(\B{\tau}^*-\V{a}_n)^{\text{T}}(\B{\lambda}^*)^++(\V{b}_n-\B{\tau}^*)^{\text{T}}(-\B{\lambda}^*)^+$$
$$=R_\Phi^{\B{\tau}_{\infty}}+\left(\B{\tau}^*-\B{\tau}_n+\V{c}\sqrt{\frac{\log m+\log\frac{2}{\delta}}{2n}}\right)^{\text{T}}(\B{\lambda}^*)^+-\left(\B{\tau}^*-\B{\tau}_n-\V{c}\sqrt{\frac{\log m+\log\frac{2}{\delta}}{2n}}\right)^{\text{T}}(-\B{\lambda}^*)^+$$
$$=R_\Phi^{\B{\tau}_{\infty}}+(\B{\tau}^*-\B{\tau}_n)^{\text{T}}\B{\lambda}^*+\sqrt{\frac{\log m+\log\frac{2}{\delta}}{2n}}\V{c}^{\text{T}}((\B{\lambda}^*)^++(-\B{\lambda}^*)^+)$$
Then the result is obtained using Cauchy-Schwarz inequality and the fact that $\|(\B{\lambda}^*)^++(-\B{\lambda}^*)^+\|_2=\|\B{\lambda}^*\|_2$. 


For the second result, note that using Proposition~\ref{prop} and since $p^*\in\set{U}^{\V{a}_n,\V{b}_n}_\Phi$ with probability at least $1-\delta$ we have that 
$$R(h^{\B{\tau}_n})\leq\max_{p\in\set{U}^{\V{a}_n,\V{b}_n}_\Phi}\ell(h^{\B{\tau}_n},p)=1-\underset{\B{\Phi}^{\text{T}}(\B{\alpha}-\B{\beta})+\gamma\preceq \V{h}^{\B{\tau}_n}}{\max}\V{a}_n^{\text{T}}\B{\alpha}-\V{b}_n^{\text{T}}\B{\beta}+\gamma$$
so that, if $\B{\lambda}_n^*,\gamma_n^*$ is a solution of \eqref{learning-eq} for $\V{a}=\B{\tau}_n$, we have that $R(h^{\B{\tau}_n})\leq 1-\V{a}_n^{\text{T}}(\B{\lambda}_n^*)^++\V{b}_n^{\text{T}}(-\B{\lambda}_n^*)^+-\gamma_n^*$ because $\B{\lambda}_n^*=(\B{\lambda}_n^*)^+-(-\B{\lambda}_n^*)^+$ and $\B{\Phi}^{\text{T}}\B{\lambda}_n^*+\gamma_n^*\preceq \V{h}^{\tau_n}$ by definition of $\V{h}^{\tau_n}$. Therefore, the result is obtained since 
\begin{align*}R(h^{\B{\tau}_n})\leq\,& 1-\left(\B{\tau}_n-\V{c}\sqrt{\frac{\log m+\log\frac{2}{\delta}}{2n}}\right)^{\text{T}}(\B{\lambda}_n^*)^++\left(\B{\tau}_n+\V{c}\sqrt{\frac{\log m+\log\frac{2}{\delta}}{2n}}\right)^{\text{T}}(-\B{\lambda}_n^*)^+\\
& -\gamma_n^*+\B{\tau}_n^{\text{T}}\B{\lambda}_n^*-\B{\tau}_n^{\text{T}}\B{\lambda}_n^*\\
=\,& R_\Phi^{\B{\tau}_n}-\B{\tau}_n^{\text{T}}\left((\B{\lambda}_n^*)^+-(-\B{\lambda}_n^*)^+\right)+\B{\tau}_n^{\text{T}}\B{\lambda}_n^*+\V{c}^{\text{T}}\sqrt{\frac{\log m+\log\frac{2}{\delta}}{2n}}\left((\B{\lambda}_n^*)^++(-\B{\lambda}_n^*)^+\right)
\end{align*}

For the third result, let 
$$\tilde{\B{\lambda}}_n,\tilde{\gamma}_n\in\arg\underset{\B{\Phi}^{\text{T}}\B{\lambda}+\gamma\preceq-\V{h}^{\tau_n}}{\max} \B{\tau}_n^{\text{T}}\B{\lambda}+\gamma$$
then $L_\Phi^{\B{\tau}_n}=1+\B{\tau}_n^{\text{T}}\tilde{\B{\lambda}}_n+\tilde{\gamma}_n$ and the result is obtained since $p^*\in\set{U}_\Phi^{\V{a}_n,\V{b}_n}$ with probability at least $1-\delta$ and hence
\begin{align}R(h^{\B{\tau}_n})&\geq 1+\underset{\B{\Phi}^{\text{T}}(\B{\alpha}-\B{\beta})+\gamma\preceq-\V{h}^{\tau_n},\B{\alpha},\B{\beta}\succeq \V{0}}{\max} \V{a}_n^{\text{T}}\B{\alpha}-\V{b}_n^{\text{T}}\B{\beta}+\gamma\geq 1+\V{a}_n^{\text{T}}(\tilde{\B{\lambda}}_n)^+-\V{b}_n^{\text{T}}(-\tilde{\B{\lambda}}_n)^++\tilde{\gamma}_n\nonumber\\
&=1+\B{\tau}_n^{\text{T}}\left((\tilde{\B{\lambda}}_n)^+-(-\tilde{\B{\lambda}}_n)^+\right)+\tilde{\gamma}_n-\V{c}^{\text{T}}\sqrt{\frac{\log m+\log\frac{2}{\delta}}{2n}}\left((\tilde{\B{\lambda}}_n)^++(-\tilde{\B{\lambda}}_n)^+\right)
\end{align}

For the fourth result, note that using Proposition~\ref{prop} and since $p^*\in\set{U}^{\B{\tau}_{\infty}}_\Phi$ we have that
$$R(h^{\B{\tau}_n})\leq\max_{p\in\set{U}^{\B{\tau}_{\infty}}_\Phi}\ell(p,h^{\B{\tau}_n})=1-\underset{\B{\Phi}^{\text{T}}\B{\lambda}+\gamma\preceq \V{h}^{\B{\tau}_n}}{\max}(\B{\tau}_{\infty})^{\text{T}}\B{\lambda}+\gamma$$
so that, if $\B{\lambda}_n^*,\gamma_n^*$ is a solution of \eqref{learning-eq} for $\V{a}=\B{\tau}_n$, we have that $R(h^{\B{\tau}_n})\leq 1-(\B{\tau}_{\infty})^{\text{T}}\B{\lambda}_n^*-\gamma_n^*$ because $\B{\Phi}^{\text{T}}\B{\lambda}_n^*+\gamma_n^*\preceq \V{h}^{\tau_n}$ by definition of $\V{h}^{\tau_n}$.  Let  $\B{\lambda}^*,\gamma^*$ be a solution of \eqref{learning-eq} for $\V{a}=\B{\tau}^*$, the result is obtained since
\begin{align}R(h^{\B{\tau}_n})&\leq 1-(\B{\tau}_{\infty})^{\text{T}}\B{\lambda}_n^*
-\gamma_n^*+\B{\tau}_n^{\text{T}}\B{\lambda}_n^*-\B{\tau}_n^{\text{T}}\B{\lambda}_n^*+(\B{\tau}_{\infty})^{\text{T}}\B{\lambda}^*+\gamma^*-(\B{\tau}_{\infty})^{\text{T}}\B{\lambda}^*-\gamma^*\nonumber\\
&=(\B{\tau}_n-\B{\tau}_{\infty})^{\text{T}}\B{\lambda}_n^*+(\B{\tau}_{\infty})^{\text{T}}\B{\lambda}^*+\gamma^*-\B{\tau}_n^{\text{T}}\B{\lambda}_n^*-\gamma_n^*+R_\Phi^{\B{\tau}_{\infty}}\nonumber\\
&\leq (\B{\tau}_n-\B{\tau}_{\infty})^{\text{T}}\B{\lambda}_n^*+(\B{\tau}_{\infty}-\B{\tau}_n)^{\text{T}}\B{\lambda}^*+R_\Phi^{\B{\tau}_{\infty}}\label{ineq1}\\
&\leq\|\B{\tau}_n-\B{\tau}_{\infty}\|_2\|\B{\lambda}_n^*-\B{\lambda}^*\|_2+R_\Phi^{\B{\tau}_{\infty}}\nonumber
\end{align}
where \eqref{ineq1} is due to the fact that $\B{\tau}_n^{\text{T}}\B{\lambda}_n^*+\gamma_n^*\geq\B{\tau}_n^{\text{T}}\B{\lambda}^*+\gamma^*$ since $\B{\lambda}^*,\gamma^*$ is a feasible point of \eqref{learning-eq} for $\V{a}=\B{\tau}_n$. 

2. Next we show the proof for $h^{\V{a}_n,\V{b}_n}$ since that for $h^{\B{\tau}_n}$ is analogous. Let $\psi$ be the function defined for $\V{a},\V{b}\in\mathbb{R}^{m}$ as $\psi(\V{a},\V{b})=R_\Phi^{\V{a},\V{b}}$, such function is concave and has subpergradients given by solutions of \eqref{learning-ineq}. Then, if $\B{\lambda}^*$ is unique solution of \eqref{learning-eq} for $\V{a}=\B{\tau}_{\infty}$, we have that $\psi$ is differentiable at $(\B{\tau}_{\infty},\B{\tau}_{\infty})$. Then, since $\V{a}_n\underset{n\to\infty}{\to} \B{\tau}_{\infty}$ and $\V{b}_n\underset{n\to\infty}{\to} \B{\tau}_{\infty}$,  if $\B{\alpha}_n^*,\B{\beta}_n^*,\gamma_n^*$ is a solution of  \eqref{learning-ineq} for $\V{a}=\V{a}_n$ and $\V{b}=\V{b}_n$, we have that $\B{\alpha}_n^*-\B{\beta}_n^*\underset{n\to\infty}{\to}\B{\lambda}^*$ and $\gamma_n^*\underset{n\to\infty}{\to}\gamma^*$ using classical results of convergence of superdifferentials  (see e.g., Proposition~2.5 in \cite{Phe:09}). Hence, the result is obtained since $h^{\V{a}_n,\V{b}_n}$ in \eqref{t-a,b} depends continuously on $\B{\alpha}_n^*,\B{\beta}_n^*,\gamma_n^*$.

\end{document}